\documentclass{ifacconf}

\usepackage{graphicx}      
\usepackage{natbib}        
\usepackage{comment}

\usepackage{amsmath,amsfonts,amssymb}
\newcommand{\R}{\mathbb{R}}

\usepackage{xcolor}

\newcommand{\C}{\mathcal{C}}
\renewcommand{\H}{\mathcal{H}}
\newcommand{\Ltwo}{\mathcal{L}^2}

\renewcommand{\L}{\mathcal{L}}

\newcommand{\F}{\mathfrak{F}}

\DeclareMathOperator*{\Argmin}{Argmin}
\DeclareMathOperator*{\Argmax}{Argmax}
\DeclareMathOperator*{\V}{\mathbb{V}}
\DeclareMathOperator*{\E}{\mathbb{E}}
\DeclareMathOperator*{\tr}{trace}
\DeclareMathOperator*{\diag}{diag}

\newtheorem{exampleEnv}{Example}
\newenvironment{example}[1][]{
  \ifx&#1&%
    \begin{exampleEnv}%
  \else
    \begin{exampleEnv}[#1]%
  \fi
}{
  \hfill$\triangle$\end{exampleEnv}
}
\newtheorem{assumption}{Assumption}
\newtheorem{remarkEnv}{Remark}
\newenvironment{remark}[1][]{
  \ifx&#1&%
    \begin{remarkEnv}%
  \else
    \begin{remarkEnv}[#1]%
  \fi
}{\hfill$\blacklozenge$\end{remarkEnv}}
\newenvironment{proof}[1][]{\begin{pf}}{\hfill$\square$\end{pf}}

\usepackage{algorithmic,algorithm}

\begin{document}
\begin{frontmatter}

\title{A Control Perspective on Training PINNs\thanksref{footnoteinfo}} 

\thanks[footnoteinfo]{This work was partially supported by the Wallenberg AI, Autonomous Systems and Software Program (WASP) funded by the Knut and Alice Wallenberg Foundation.}

\author[First]{Matthieu Barreau} 
\author[First]{Haoming Shen} 
            
\address[First]{Digital Futures and KTH Royal Institute of Technology, Stockholm, Sweden (e-mail: {barreau,haomings}@kth.se).}

\begin{abstract}                
We investigate the training of Physics-Informed Neural Networks (PINNs) from a control-theoretic perspective. Using gradient descent with resampling, we interpret the training dynamics as asymptotically equivalent to a stochastic control-affine system, where sampling effects act as process disturbances and measurement noise. Within this framework, we introduce two controllers for dynamically adapting the physics weight: an integral controller and a leaky integral controller. We theoretically analyze their asymptotic properties under the accuracy–robustness trade-off and we evaluate them on a toy example. Numerical evidence suggests that the integral controller achieves accurate and robust convergence when the physical model is correct, whereas the leaky integrator provides improved performance in the presence of model mismatch. This work represents a first step toward convergence guarantees and principled training algorithms tailored to the distinct characteristics of PINN tasks.
\end{abstract}


\end{frontmatter}

\section{Introduction}

Physics-Informed Neural Networks (PINNs) have emerged as a powerful deep learning paradigm for learning solutions to differential equations and characterizing complex dynamical behaviors~\citep{karniadakis2021physics}. The key advantages lie in the incorporation of physical laws into the minimization of the loss function, which can be interpreted as a regularization agent that enforces physical constraints via penalty terms, enabling efficient learning even in the presence of limited available data. However, like any deep learning methods, PINNs inherit stochastic properties from their underlying architecture. Moreover, PINNs are formulated as multi-objective optimization problems that involve multiple loss components (e.g., losses associated with physical residuals, initial conditions, and boundary conditions), in which unbalanced gradient flow can lead to instability during training~\citep{wang2020understanding}. These issues pose barriers to achieving robust and efficient training, particularly for large-scale or complex systems.

Improving the accuracy and robustness of neural network training has attracted significant attention in the Machine Learning community. A successful example is the Xavier initialization, which has been shown to substantially improve training stability, particularly for large and deep neural networks \citep{glorot2010understanding}. Meanwhile, various optimization algorithms such as Adam have demonstrated strong empirical performance in accelerating convergence and enhancing robustness \citep{kingma2015adam}. In addition, growing research interest has focused on developing methods specifically tailored to improve PINNs performance, including pretraining \citep{guo2023pre,li2025milpinitializationsolvingparabolic}, reformulations of the underlying mathematical problem \citep{sirignano2018dgm,son2023sobolev}, novel architectures \citep{shukla2021parallel,jagtap2020adaptive}, and new learning paradigms such as meta-learning and curriculum learning \citep{psaros2022meta,liu2024config}. 

One line of studies focuses on adaptively adjusting the weights associated with each loss term to mitigate gradient imbalance and numerical stiffness during training. Successful implementations include rescaling the data loss relative to the physics loss based on observed gradient magnitudes \citep{wang2020understanding}, or utilizing neural tangent kernel analysis for PINNs to ensure similar convergence rates for all losses \citep{wang2020and}. \cite{maddu2022inverse} propose a variance-based reweighting scheme by monitoring the variance of gradients and assigning inversely to their uncertainty, which helps prevent gradient vanishing. By introducing trainable weights, PINNs can adaptively penalize hard-to-minimize loss components during the training process~\citep{mcclenny2023self}.

Despite existing reweighting strategies that improve PINNs' performance, the training process can exhibit significant variability and sensitivity to uncertainty due to the intrinsic stochastic optimization dynamics inherent in batch training, initialization, and the complex architectures of PINNs. Moreover, the proposed approaches are empirical and often lack interpretations. In contrast, a rich body of research exists in the control community, focusing on ensuring stability and robustness in system dynamics under uncertainty.

Motivated by~\cite{cerone2025new}, where the learning under constraints problem is framed as a dynamical system, our contribution is to reinterpret PINN training as a control-affine system and analyze controller choices. In Section~2, we introduce some background and formulate the problem. Then, we propose investigating certain properties of the optimization problem to establish the existence of a solution and characterize the equilibrium points in Section~3. Section~4 describes the gradient descent algorithm in the form of a noisy dynamical system. Section 5 presents the main contribution, where we study two basic control schemes: an integral control for achieving zero steady-state error and a leaky integrator for enhanced robustness. Finally, we investigate a toy example in Section~6 to showcase the potential of the approach. Section~7 concludes the work with a summary and perspectives.

\emph{Notations:} We use $(\R^n, \| \cdot \|)$ as the real normed vector space. For a finite set $\Omega \subset \R^n$, $| \Omega |$ refers to its cardinal, and in the case of a compact set, $|\Omega|$ is its volume. We define the following functional spaces: $\Ltwo(A, B)$ as the space of squared integrable functions from $A$ to $B$, $\mathcal{H}^q(A, B) \subseteq \Ltwo(A, B)$ is the Sobolev space of functions from $A$ to $B$ which have weak-derivatives up to the order $q$. If $A$ is a compact subset of $\R^n$, the norm used in $\mathcal{H}^q(A, B)$ is 
\[
    \| v \|_{\mathcal{H}^q}^2 = \sum_{k=0}^q \sum_{|\alpha| = k} \int_A \| D^{\alpha} v(u) \|^2 du
\]
where $\alpha = (\alpha_{1}, ..., \alpha_{n})$ is a multi-index of order $|\alpha|=k$ and $D^{\alpha\!}f = \frac{\partial^{| \alpha |}\! f}{\partial x_{1}^{\alpha_{1}} \dots \partial x_{n}^{\alpha_{n}}}$.

\section{Problem Formulation}

Let $\F: \H^q(U,\R^n) \to \Ltwo(U, \R^n)$ be a possibly nonlinear operator where $U$ is an open bounded subset of $\R^p$. We consider the ODE/PDE problem
\begin{equation} \label{eq:PDE}
    \begin{array}{ll}
        \displaystyle \F[v](u) = 0, & u \in U, \\
        v(s) = v_{\Gamma}(s), & s \in \Gamma_1 \cup \Gamma_2, \\
    \end{array}
\end{equation}
where $v: U \to \R^n$ belongs to a Sobolev space $\H^q(U,\R^n)$ with $q \geq 0$, $\Gamma_1$ is a subset of the boundary of $U$ and $v_{\Gamma}$ is prescribed as the initial and boundary conditions. Additional in-domain point-wise knowledge is provided through $\Gamma_2 \subset U$.

\begin{example}
    Consider the following autonomous ordinary differential equation:
    \begin{equation}
        \label{eq:ODE}
        \dot{x}(t) = f(x(t)), \quad t \geq 0,
    \end{equation}
    with $x(0) = x_0 \in \mathbb{R}^n$. This fits the framework in \eqref{eq:PDE} with $v = x$, $q = 1$, $U = [0, T]$ for any $T > 0$, $\Gamma_1 = \{0\}$ with $v_{\Gamma}(0) = x_0$, and 
    \[
        \forall u \in [0, T], \quad \F[v](u) = \dot{v}(u) - f(v(u)).
    \]
\end{example}

\begin{example}
    Consider the following second-order partial differential equation:
    \begin{multline}
        \label{eq:heat}
        \partial_t v(t,x) = - \gamma\left( t, x, v(t,x), \nabla_x v(t,x) \right) \\
        + \textrm{div}\left( \alpha(t, x, v(t,x), \nabla_x v(t,x) \right)
    \end{multline}
    for $t, x \in [0, T] \times \Omega \subset \R^n$ with $v(0,\cdot) = v_0$, $v(t,x) = g(t,x)$ for $t \in (0, T]$ and $x \in \partial \Omega$. We can define the operator $\F[v] = \partial_t v + \gamma - \textrm{div} \ \alpha$ and $U = (0,T]\times\Omega$.
\end{example}

Assuming there is a unique weak-solution $v$ to \eqref{eq:PDE} with $\Gamma_2 = \emptyset$, the \emph{forward} problem in PINNs is to approximate $v$ using a deep feed-forward neural network $v_{\Theta}$ with $L$ hidden layers and a linear output defined by
\begin{equation} \label{eq:nn}
    v_{\Theta}(u) = W_{L+1} \times H_{L} \circ \cdots \circ H_1 (u) + b_{L+1}, \quad \quad u \in U
\end{equation}
where $H_i(u) = \phi(W_i u + b_i)$, $\phi \in \C^{\infty}(\R, \R)$ being an element-wise activation function, $b_i$ and $W_i$ are real matrices of appropriate dimensions. The number of neurons in each layer is denoted $N \in \mathbb{N}$. The function $v_{\Theta}$ is parametrized by the tensor $\Theta = \{b_1, W_1, \dots, b_{L+1}, W_{L+1}\}$. The \emph{inverse} problem is an identification task where we try to learn $\F$ such that $v_{\Theta} - v_{\Gamma_2}$ is also minimized on $\Gamma_2$. 

Approximating a weak solution to the original problem~\eqref{eq:PDE} is equivalent to solving the following one on $v_{\Theta}$:
\begin{equation} \label{eq:optimization_constraint}
    \Theta_* \in \begin{array}[t]{cl} \displaystyle \Argmin_{\mathbf{\Theta}} & \displaystyle \int_{\Gamma} \| v_{\Gamma}(s) - v_{\mathbf{\Theta}}(s) \|^2 ds \\
    \text{ s.t. } & \displaystyle \int_U \| \F[v_{\mathbf{\Theta}}](u) \|^2 du = 0 .
    \end{array}
\end{equation}
The vanilla PINN framework \citep{karniadakis2021physics} discretizes the sets $\Gamma$ and $U$ into $D_{\Gamma}$ and $D_U$ and then creates the data and physics costs as
\[
    \mathcal{L}_{\textrm{data}}(\Theta, D_\Gamma) = \frac{1}{\left| D_{\Gamma} \right|}\sum_{s \in D_{\Gamma}} \left\| v(s) - v_{\Theta}(s) \right\|^2,
\]
\[
    \mathcal{L}_{\textrm{phys}}(\Theta, D_U) = \frac{1}{\left| D_{U} \right|}\sum_{u \in D_{U}} \left\| \F[v_{\Theta}](u) \right\|^2.
\]
The data cost is the mean squared error between the measurement points and the neural network $v_{\Theta}$ evaluated at the same locations, and the physics one is the mean squared error between $\mathfrak{F}[v_{\Theta}]$ and zero on the discrete set $D_U$. By considering the physics cost as a regularization agent, it leads to the extended discretized cost:
\begin{equation} \label{eq:L_lambda}
    \L_{\lambda}^d(\Theta, D_\Gamma, D_U) = \L_{\textrm{data}}(\Theta, D_\Gamma) + \lambda \L_{\textrm{phys}}(\Theta, D_U)
\end{equation}
where $\lambda > 0$, and \eqref{eq:optimization_constraint} is relaxed to
\begin{equation}
    \label{eq:PINN}
    \Theta_* = \Argmin_{\mathbf{\Theta}} \L_{\lambda}^d(\mathbf{\Theta}, D_\Gamma, D_U).
\end{equation}

\begin{remark}
    Technically speaking, if $\mathfrak{F}[v_{\Theta}]$ contains derivatives of $v_{\theta}$, one can use automatic differentiation to calculate it appropriately \citep{baydin2017automatic}.
\end{remark}

A strategy to approximate a solution to \eqref{eq:PINN} is to use a gradient-descent algorithm, which writes as:
\begin{equation}
    \label{eq:gradient-descent}
    \Theta_{k+1} = \Theta_k - \alpha(k) \nabla_{\Theta} \L_{\lambda}^d(\Theta_k, D_\Gamma, D_U), \quad \Theta(0) \sim G
\end{equation}
where $\alpha > 0$ is the learning rate and $G$ is the Xavier-Glorot initialization \citep{glorot2010understanding}. Under some conditions on the loss $\mathcal{L}_{\lambda}$ and the learning rate $\alpha$, it is possible to prove convergence of \eqref{eq:gradient-descent}. In these cases, $\Theta$ has a limit $\Theta_*$ and $\nabla_{\Theta} \L_{\lambda}^d(\Theta_*) = 0$. That means the algorithm reached a local optimum. However, this local optimum might be far from the global optimum, and the following generalization error might be large:
\begin{equation}
    \label{eq:generalization_error}
    \mathcal{E}(\Theta) = \int_U \left\| v(u) - v_{\Theta}(u) \right\|^2 du.
\end{equation}
Therefore, a training algorithm $\mathcal{T}$ is evaluated over two quantities:
\begin{equation*}
    \label{eq:accuracy_robustness}
    A(\mathcal{T}) = \E_{\Theta_0 \sim G} \left\{  \mathcal{E} \left( \mathcal{T}(\Theta_0) \right)\right\}, \quad R(\mathcal{T}) = \V_{\Theta_0 \sim G} \left\{  \mathcal{E} \left( \mathcal{T}(\Theta_0) \right)\right\}.
\end{equation*}

The stochasticity of the training algorithm, due to its internal processes and the initial condition, necessitates considering the average and variance of the error as accuracy $A$ and robustness $R$ metrics, respectively. Increasing the robustness of a training algorithm (small $R$) is desirable to ensure that its accuracy accurately reflects its efficiency. 

In the case $\Gamma_2 \neq \emptyset$, there are two different regimes possible for the accuracy. Either $\| v_{\Gamma} - v_{\Theta_*}\|$ on $\Gamma_2$ is minimal; this is the \emph{match} case. Alternatively, there may be a discrepancy, and the additional measurements contradict the physical model; this is the \emph{mismatch} case. Solving the forward problem in both these cases is the setup of this paper.

\textbf{Problem Statement:} We want to design a robust, accurate, and simple algorithm for training a PINN in the match and mismatch cases.

\section{Preliminaries}

All the above definitions assume that there is a unique solution to compare with, and that \eqref{eq:optimization_constraint} admits a global minimum. To ensure these two key properties, we make the following assumptions.

\begin{assumption}
    We assume that problem~\eqref{eq:PDE} is well-posed with $\Gamma_2 = \emptyset$, meaning that there exists a unique solution $v: U \to \R^n$ which satisfies Equation~\eqref{eq:PDE} in a weak sense.
\end{assumption}

\begin{assumption}
    The operator $\F$ is globally Lipschitz continuous.
\end{assumption}

The first assumption is required to infer a solution, while the second one is a classical assumption for analysis. Note that, even if very conservative, the second assumption can be relaxed to local Lipschitz continuity if the solution stays bounded \citep{Hartman2002ODE}.  

\begin{example}[Ex.1 continued]
    Let $\F[v](u) = \dot{v}(u) - f(x(u))$ and assume that $f$ is globally Lipschitz continuous with constant $L_f$. Then, by Picard-Lindelöf existence theorem, there exists a unique solution $v \in \H^1(\R^+, \R^n)$ to \eqref{eq:ODE}. Therefore, Assumption~1 is verified.
    
    Moreover, we get that for all $v, \hat{v} \in \H^1$:
    \begin{multline*}
        \int_0^T \| \F[\hat{v}](t) - \F[v](t) \|^2 dt\\
        \leq 2 \int_0^T \left( \| \dot{v}(t) - \dot{\hat{v}}(t) \|^2 + \| f(\hat{v}(t)) - f(v(t)) \|^2\right)dt \\
        \leq 2 \int_0^T \left( \| \dot{v}(t) - \dot{\hat{v}}(t) \|^2 + L_f^2 \| \hat{v}(t) - v(t) \|^2\right)dt \\
        \leq 2 \max\left(1, L_f^2 \right) \| \hat{v} - v\|_{\H^1}^2.
    \end{multline*}
    Consequently, $\F$ is a globally Lipschitz operator, and Assumption~2 is verified.
\end{example}

\begin{example}[Ex.2 continued]
    Let $\F[v] = \partial_t v + \gamma - \textrm{div} \ \alpha$ and $U = (0,T]\times\Omega$. Then, under some conditions on the domain $\Omega$, and on the functions $\alpha$ and $\gamma$, there exists a unique solution to \eqref{eq:heat} \citep{ladyzhenskaia1968linear}. Moreover, under some other conditions (such as polynomial growth) on $\alpha$ and $\gamma$, the operator $\F$ is a global Lipschitz operator \cite[Theorem~7.1]{sirignano2018dgm}. In that case, Assumptions~1 and 2 are verified.
\end{example}

The two assumptions are made to ensure that the neural-network approximation $v_{\Theta}$ proposed in \eqref{eq:nn} yields a global minimum loss of zero. Using the Lagrangian formulation, an optimal solution to \eqref{eq:optimization_constraint} would be a minimizer of the following loss
\begin{equation}
    \L_{\lambda}(\Theta) = \!\int_{\Gamma} \| v_{\Gamma}(s) - v_{\Theta}(s) \|^2 ds \: + \: \lambda \! \int_U \| \F[v_{\Theta}](u) \|^2 du
\end{equation}
with $\lambda > 0 $.

\begin{prop}
    \label{prop:minimum_value}
    Assume that $L >0, \Gamma_2 = \emptyset, \lambda > 0$, then the minimal value of $\L_{\lambda}$ can be made as small as desired, provided enough neurons $N$.
\end{prop}

\begin{proof}
    Assumption~1 guarantees that there exists $v \in \mathcal{H}^q$ such that $\F[v] = 0$. From Assumption~2, there exists $L > 0$ such that
    \[
        \forall v, \hat{v} \in \mathcal{H}^q, \quad \int_U \| \mathfrak{F}[\hat{v}](u) - \mathfrak{F}[v](u) \|^2 du \leq L^2 \| \hat{v} - v \|_{\mathcal{H}^q}^2.
    \]
    Using the previous inequality and the fact that $v_{\Theta} \in \mathcal{H}^p$ as well, we get
    \begin{equation}
        \label{eq:inequality_L}
        \begin{array}{rl}
            \L_{\lambda}(\Theta) &= \int_{\Gamma} \| v_{\Gamma}(s) - v_{\Theta}(s) \|^2 ds \hspace{3cm}\\
            & \hfill + \lambda \! \int_U \| \F[v_{\Theta}](u) - \mathfrak{F}[v](u) \|^2 du \\
            &\leq \int_{\Gamma} \| v_{\Gamma}(s) - v_{\Theta}(s) \|^2 ds + \lambda L^2 \| v_{\Theta} - v \|_{\mathcal{H}^p}^2.
        \end{array}
    \end{equation}
    Under the conditions of the theorem, the universal approximation theorem by \cite{hornik1991approximation} ensures that for any $\varepsilon > 0$, there exists a number of neurons $N > 0$ such that
    \[
        \max\left(\sup_{s \in \Gamma} \| v_{\Gamma}(s) - v_{\Theta}(s) \|, \sup_{u \in U} \max_{|\alpha| \leq q} \|D^\alpha v (u)\| \right) < \varepsilon.
    \]
    Plugging this back into \eqref{eq:inequality_L} leads to the existence of $C > 0$ as a function of $p$ and $q$ such that
    \[
        \L_{\lambda}(\Theta) \leq \left( |\Gamma| + C \lambda |U| L^2 \right) \varepsilon^2
    \]
    which ends the proof.
\end{proof}

Note that, in some cases, if there exists $\Theta_*$ such that $\L_{\lambda}(\Theta_*) = 0$ then $v_{\Theta_*}$ is the unique solution to \eqref{eq:PDE}. This phenomenon occurs for certain classes of parabolic equations \citep[Theorem~7.3]{sirignano2018dgm} and also for finite-dimensional dynamical systems, as demonstrated below.

\begin{prop}
    Assume the ODE case of Example~1. Then Proposition~\ref{prop:minimum_value} holds and for $\lambda > 0$ and $|\Gamma| > 0$, $\L_{\lambda}(v) = 0$ is equivalent to $v$ is the unique solution to \eqref{eq:PDE}.
\end{prop}

\begin{proof}
    Because Assumptions~1 and 2 hold, then Proposition~\ref{prop:minimum_value} also holds. Moreover, $\L_{\lambda}(v) = 0$ implies that both contributions in $\L_{\lambda}$ must be zero and there exists $s \in \Gamma$ such that $v(s) = 0$ with $\F[v] = 0$ almost everywhere. This is a Cauchy problem, and $v = v^*$ by the uniqueness property. 
\end{proof}

We have seen that if we find parameters $\Theta_*$ which are globally minimizing the Lagrangian $\L_{\lambda}$, then, provided some conditions on $\F$, $v_{\Theta_*}$ is an approximation to $v$. This means that there is a low generalization error~\eqref{eq:generalization_error}. The link between the generalization error and a small Lagrangian is provided. However, gradient-descent algorithms usually only find local optima. Some results in machine learning strongly indicate that local minima tend to converge to the global minimum, provided a large number of neurons \citep{choromanska2015loss}. This suggests that using a large neural network would increase the accuracy and robustness of gradient-descent algorithms by decreasing the probability of reaching a "bad" local minimum. We will investigate strategies to strengthen this claim in the following sections.

\section{Control Formulation}

The training algorithm proposed to estimate a solution to \eqref{eq:optimization_constraint} is given in Algorithm~\ref{algo:PD}. This is the standard gradient-descent procedure (Line~7) with additional steps such as the resampling sub-routines (Lines~3-6) together with an adaptive weight $\lambda$ that is computed as the output of the function \texttt{compute\_control}.

\begin{algorithm}
    \caption{Learning algorithm for \eqref{eq:optimization_constraint}.}
    \label{algo:PD}
    \begin{algorithmic}[1]
        \STATE{Initialize $\Theta_0 \sim G$ randomly, $\lambda_0 = 0$}
        \FOR{$k$ from $0$ to $N_{\textrm{epoch}}$}
            \IF{$k \bmod N_s = 0$}
                \STATE{Sample $N_\Gamma$ points uniformly in $\Gamma$ to form $D_\Gamma$}
                \STATE{Sample $N_{U}$ points uniformly in $U$ to form $D_U$}
            \ENDIF
            \STATE{$\Theta_{k+1} = \Theta_k - \alpha(k) \nabla_{\Theta} \mathcal{L}_{\lambda_k}(\Theta_k, D_\Gamma, D_U)$}
            \STATE{$\displaystyle\lambda_{k+1} = \mathtt{compute\_control}(\lambda_k, \Theta_k, D_U, D_\Gamma)$}
        \ENDFOR
        \RETURN $\Theta_k$
    \end{algorithmic}
\end{algorithm}

The dynamical system formulation of an algorithm is an effective way to analyze its asymptotic behavior. By setting $\xi = k \alpha$ and then making $\alpha$ go to $0$. Under proper conditions derived by \cite{elkabetz2021continuous}, Line~7 of Algorithm~\ref{algo:PD} is equivalent to:
\begin{equation}
    \label{eq:CT_gradient_descent}
    \tfrac{d}{d\xi}\Theta(\xi) = - \nabla_{\Theta} \L_{\lambda(\xi)}^d\left(\Theta(\xi), D_\Gamma, D_U\right).
\end{equation}
We then leverage the framework proposed by \cite{cerone2025new} to reformulate the learning under constraints problem \eqref{eq:optimization_constraint} together with the gradient descent scheme~\eqref{eq:CT_gradient_descent} as an input-output dynamical system:
\begin{equation}
    \label{eq:dynamical_system}
    \left\{
    \begin{array}{l}
        \tfrac{d}{d\xi}\Theta(\xi) = - f^d(\Theta(\xi)) - g^d(\Theta(\xi)) u(\xi), \\
        y(\xi) = h^d(\Theta(\xi)), \\
        \Theta(0) \sim G,
    \end{array}
    \right.
\end{equation}
where $u = \lambda$ is the input, $y$ is the output, and 
\[
    \begin{array}{ll}
        f^d(\Theta) = \nabla_{\Theta} \L_{\textrm{data}}(\Theta, D_\Gamma), &\quad g^d(\Theta) = \nabla_{\Theta} \L_{\textrm{phys}}(\Theta, D_U), \\
        h^d(\Theta) = \L_{\textrm{phys}}(\Theta, D_U).
    \end{array}
\]
This is a classical control-affine system with state $\Theta$, measurement $y$, and control action $\lambda$. Within this framework, however, we do not take into account the resampling. A solution is to consider Monte-Carlo sampling to estimate the integrals in \eqref{eq:optimization_constraint} \citep[Chap.~17]{Goodfellow-et-al-2016}.
By the central limit theorem, we get the following equalities:
\[
    \begin{array}{l}
        f^d (\Theta) = \left\{\frac{1}{| \Gamma |}\int_{\Gamma}  \nabla_{\Theta} \left\| v - v_{\Theta} \right\|^2 \right\} + n_f \triangleq f(\Theta) + n_f, \\
        g^d(\Theta) = \left\{ \frac{1}{|U|} \int_U \nabla_{\Theta} \left\| \F[v_{\Theta}] \right\|^2 \right\} + n_g \triangleq g(\Theta) + n_g, \\
        h^d(\Theta) = \left\{\frac{1}{|U|} \int_U \left\| \F[v_{\Theta}] \right\|^2\right\} + n_h \triangleq h(\Theta) + n_h.
    \end{array}
\]
where $n_f, n_g, n_h$ are unbiased Gaussian noise with respective covariance
\[
    Q_{f} \!=\!\! \tfrac{\V\left\{\nabla_{\Theta} \left\| v - v_{\Theta} \right\|^2\right\}}{|D_\Gamma|}, Q_{g} \!=\!\! \tfrac{\V\left\{\nabla_{\Theta} \left\| \F[ v_{\Theta}] \right\|^2\right\}}{|D_U|}, Q_{h} \!=\!\! \tfrac{\V\left\{\left\| \F[v_{\Theta}] \right\|^2\right\}}{|D_U|}.
\]

Plugging this back into \eqref{eq:dynamical_system} leads to :
\begin{equation}
    \label{eq:dynamical_system_resampling}
    \left\{
    \begin{array}{l}
        \tfrac{d}{d\xi}\Theta(\xi) = - f(\Theta(\xi)) - g(\Theta(\xi)) u(\xi) + w(\xi), \\
        y(\xi) = h(\Theta(\xi)) + n(\xi), \\
        \Theta(0) \sim G,
    \end{array}
    \right.
\end{equation}
with unbiased disturbance $w$ and noise $n$ with covariances $Q_w = Q_{f} + u Q_g$ and $Q_h$. This new formulation integrates the resampling subroutine into the algorithm as stochastic elements, showing that the proposed algorithm behaves asymptotically as a stochastic differential equation. 
A robust and accurate algorithm for training a PINN turns into designing a controller $\xi \mapsto u(y(\xi))$ ensuring ultimate-boundedness of \eqref{eq:dynamical_system_resampling}.

\begin{remark}
    Compared to a classical control problem, both the disturbance and the noise have variances that change during training. Additionally, the disturbance variance also depends on the control input. 
\end{remark}

\section{Controller Design}

In this section, we investigate classic integral controllers for dynamical system~\eqref{eq:dynamical_system_resampling}.

\subsection{Primal-Dual optimization / Integral controller}

The Lagrangian theory \citep{bertsekas2014constrained} rewrites discretized formulation of \eqref{eq:optimization_constraint} to the following equivalent optimization problem:
\[
    \Theta_* \in \Argmin_{\Theta} \max_{\lambda > 0} \L^d_{\lambda}(\Theta).
\]
An approximated solution is given by the primal-dual approach, which consists of sequentially solving these two problems:
\begin{equation} \label{eq:primal}
    \Theta_{\star}(\lambda) \in \Argmin_{\Theta} \mathcal{L}^d_{\lambda}(\Theta, D_{\Gamma}, D_U), \quad \quad \quad \text{(Primal)}
\end{equation}
and
\begin{equation} \label{eq:dual}
    \lambda_*(\Theta) \in \Argmax_{\lambda} \mathcal{L}^d_{\lambda}(\Theta, D_{\Gamma}, D_U).  \quad \quad \quad \text{(Dual)}
\end{equation}
Instead of this infinite sequential training, the parallel solution, which alternates between solving the primal and the dual to get close to a local optimum of the original constrained problem, has been proposed by \cite{goemans1997primal,fioretto2020lagrangian,nandwani2019primal} in other settings.

Here, we can apply this modified algorithm, resulting in a gradient descent update of the primal problem followed by a gradient ascent on the dual to ensure that the constraints are satisfied. This leads to this formulation:
\[
    \begin{array}{ll}
        \Theta_{k+1} = \Theta_k - \alpha \nabla_{\Theta} \L^d_{\lambda_k}(\Theta_k, D_{\Gamma}, D_U), & \Theta_0 \sim G,  \\
        \lambda_{k+1} = \lambda_k + k_I \alpha \nabla_{\lambda} \L^d_{\lambda_k}(\Theta_{k+1}, D_{\Gamma}, D_U), \quad \quad & \lambda_0 = 0.
    \end{array}
\]

Note that, in our case, we get:
\[
    \nabla_{\lambda} \L^d_{\lambda_k}(\Theta_{k+1}, D_{\Gamma}, D_U) = \L_{\textrm{phys}}(\Theta_{k+1}, D_U) = y(k\alpha).
\]
Therefore, the dual update can be simplified to:
\begin{multline*}
    \lambda_{k+1} = \texttt{compute\_control}(\lambda_k, \Theta_k, D_U, D_{\Theta}) \\
    =  \lambda_k + \alpha k_I y(k\alpha).
\end{multline*}
Making $\alpha \to 0$ leads to $u(\xi) = \lambda(\xi) = k_I \int_0^\xi y(s)ds$, which is an integral controller for \eqref{eq:dynamical_system_resampling}. Since $y, k_I \geq 0$, then $\lambda \geq 0$ and Proposition~\ref{prop:minimum_value} still applies with a global minimum.

This interpretation of the primal-dual algorithm as a continuous-time integrator has been recently investigated by \cite{qu2018exponential}, for instance; however, it has never been applied to the training of PINNs, which exhibits a stochastic behavior. As expected, when using an integral controller, this method ensures that no static error ($y \to 0$) occurs in the case of convergence. Under some classical assumptions (such as convexity, infinite neural network width, and small learning rates $\alpha$), Theorem~2.2 by \cite{daskalakis2018limit} mathematically states the almost surely convergence to $y = 0$, ensuring that $\F[v_{\Theta}]$ is near $0$. 

From a control-theoretic perspective, an integral control ensures mitigation of the effects of unbiased noise and disturbances \citep{aastrom1995pid}, along with a reduced steady-state error, which is crucial for improved accuracy and robustness.

\begin{remark}[Other interpretations]
    Note that the integral controller could also be understood as an adversarial learning strategy, where we increase the loss (dual update) to ensure that we find the optimal parameters $\Theta_*$ in the worst case \citep{goodfellow2014generative}. This is strengthening the robustness of the approach. 
    
    Another interpretation of this control method comes from curriculum learning \citep{bengio2009curriculum}: starting from $\lambda = 0$ ensures that we first fit the measurements, then $\lambda$ increases, and we start fitting the physics. Since the physics loss $\L^d_{\textrm{phys}}$ might have steep variations, it is poorly conditioned \citep{rathore2024challenges} and a large $\lambda$ would result in a similar conditioning for $\L^d_{\lambda}$. However, moving from a low $\lambda$ to a larger one ensures that we stay in a relatively "flat" region for the physics cost and that we grow in complexity with $\xi$.
\end{remark}


\subsection{Leaky Integral controller}

It is well known that the integral controller may blow up in certain nonlinear cases, and that anti-windup strategies should be considered \citep{zaccarian2011modern}. One solution, investigated by \cite{weitenberg2018robust} is to consider an integral with a forgetting factor $\phi$, written as:
\begin{equation}
    \label{eq:forgetting_integral}
    u(\xi) = k_I \int_0^\xi y(s) \exp(- \phi (\xi-s)) ds.
\end{equation}
This modified controller is designed to prevent boundedness of the control signal (unless $y$ diverges exponentially fast) and will provide additional robustness through an improved stability margin. Indeed, the steady-state error might be strictly positive, and the equilibrium points will be ultimately bounded. Moreover, for a large forgetting factor $\phi$, it approximates a proportional controller with probably faster convergence but no constraint on the static gain.

Differentiating Equation~\eqref{eq:forgetting_integral} leads to the following leaky differential equation:
\begin{equation}
    \label{eq:leaky_controller}
    \left\{
        \begin{array}{ll}
            \tfrac{d}{d\xi}I(\xi) = - \phi I(\xi) + y(s), \quad \quad & I(0) = 0, \\
            u(\xi) = k_I I(\xi).
        \end{array}
    \right.
\end{equation}
This is a convenient form that includes the pure-integral case ($\phi = 0$) and leads to the following discretized controller:
\begin{multline*}
    \lambda_{k+1} = \texttt{compute\_control}(\lambda_k, \Theta_k, D_U, D_{\Gamma}) \\
    = (1 - \alpha k_I \phi)\lambda_k + \alpha k_I y(k\alpha).
\end{multline*}

Note that since $y, k_I \geq 0$, then $\lambda \geq 0$ as long as $0 \leq \phi k_I \leq \alpha^{-1}$ and Proposition~\ref{prop:minimum_value} still applies with the zero-steady state as a global minimum.

\begin{remark}
    The leaking integrator also mitigates the effect of noise, which has a varying standard deviation. As the mean value converges, the standard deviation remains relatively constant, and the initial varying steps are discarded by the forgetting factor.
\end{remark}



\subsection{Convergence properties}

First, we investigate the equilibrium points and their stability before stating some practical consequences.

\subsubsection*{Local stability analysis}

\begin{prop}
    An equilibrium point $(\Theta_*, I^*)$ of the deterministic closed-loop system satisfies
    \begin{itemize}
        \item $f(\Theta_*) = - \frac{k_I}{\phi} g(\Theta_*) h(\Theta_*)$, $I^* = \phi^{-1} h(\Theta_*)$ if $\phi > 0$;
        \item $f(\Theta_*) = h(\Theta_*) = 0$, $I^* \in [0, \infty)$ if $\phi = 0$.
    \end{itemize}
\end{prop}

\begin{proof}
    In the case $\phi > 0$, the first equation comes straightforwardly. For the second item, one must note first that, if $\phi = 0$ then we get $h(\Theta_*) = 0$, which in turn implies $\| \F[v_{\Theta_*}] \|$ is zero almost everywhere. Consequently $g(\Theta_*) = 0$ which leads to $f(\Theta_*) = 0$.
\end{proof}

\begin{remark}
    If the closed-loop system converges in the case $\phi = 0$, then $\Theta$ converges to a strict subset of the equilibrium points of the system with constant $u > 0$ (vanilla PINN).
\end{remark}

Around an equilibrium point, we can linearize the system, yielding the following local stability properties for both controllers.

\begin{prop}
    \label{prop:equilibrium}
    Around an equilibrium point $(\Theta_*, I^*)$ of \eqref{eq:dynamical_system_resampling}-\eqref{eq:leaky_controller} with $\phi = 0$ where $f$ is locally strongly convex, we get:
   \[
        \begin{array}{ll}
            \E(\Theta(\xi)) \xrightarrow[\xi \to \infty]{} \Theta_*, \quad & \E\left(\|\Theta(\xi) - \Theta_*\|^2\right) \xrightarrow[\xi \to \infty]{} \tr(S), \\
            \E(I(\xi)) \xrightarrow[\xi \to \infty]{} I^*, & \E\left(|I(\xi) - I^*|^2\right) \xrightarrow[\xi \to \infty]{} Q_h,
        \end{array}
    \]
    where $S$ is the solution to $S \nabla_{\Theta}^2 f(\Theta_*) + \nabla_{\Theta}^2 f(\Theta_*) S = Q_w$.
\end{prop}
\begin{proof}
    Let $z = \Theta - \Theta_*$ and $S(\xi) = \E\left(z(\xi)^{\top} z(\xi)\right)$. If we start close to the equilibrium, then $z$ is small and follows the linearized dynamics:
    \[
        \tfrac{d}{d\xi}z \simeq -\nabla_{\Theta}^2 f(\Theta_*) z + w, \quad \tfrac{d}{d\xi} I \simeq \nabla_{\Theta} h(\Theta_*) z + n.
    \]
    Using the same reasoning as in the previous proof, we get $\nabla_{\Theta} h(\Theta_*) = 0$ and therefore $\tfrac{d}{d\xi} I = n$. Using It{\^o} calculus \citep{liu2019stochastic}, we get:
    \[
        \tfrac{d}{d\xi} S(\xi) = -S(\xi) \nabla_{\Theta}^2 f(\Theta_*) - \nabla_{\Theta}^2 f(\Theta_*) S(\xi) + Q_w.
    \]
    Since the loss is locally strongly convex, that means that $-\nabla_{\Theta}^2 f(\Theta_*)$ is symmetric negative definite, then all its eigenvalues are strictly negative, ensuring that $\frac{d}{d\xi} S(\xi) \to 0$, which concludes the proof.
\end{proof}

In the leaky case, the following proposition holds.
\begin{prop}
    Around an equilibrium point $z^* = \left[ \begin{matrix} \Theta_* & I^* \end{matrix}\right]^{\top}$ of \eqref{eq:dynamical_system_resampling}-\eqref{eq:leaky_controller} and $\phi > 0, k_I = 1$ where $f + h(\Theta_*) \phi^{-1} h$ is locally strongly convex, then for $z = \left[ \begin{matrix} \Theta & I \end{matrix}\right]^{\top}$, we get:
   \[
        \E(z(\xi)) \xrightarrow[\xi \to \infty]{} z^*, \quad \E\left(\|z(\xi) - z^*\|^2\right) \xrightarrow[\xi \to \infty]{} \tr(S),
    \]
    where $S$ is the solution to $S J + J^{\top} S = -\diag(Q_w, Q_h)$ with
    \[
        J = \left[ \begin{matrix} - \nabla_{\Theta}^2 f(\Theta_*) - \frac{h(\Theta_*)}{\phi} \nabla_{\Theta}^2 h(\Theta_*) & g(\Theta_*) \\ g(\Theta_*)^{\top} & - \phi \end{matrix} \right].
    \]
\end{prop} 
\begin{proof}
    Linearizing the deterministic dynamical system around the equilibrium point $z^*$ leads to:
    \[
        \tfrac{d}{d\xi}z \simeq \left[ \begin{matrix} - \nabla_{\Theta}^2 f(\Theta_*) - \nabla_{\Theta} g(\Theta_*) I^* & g(\Theta_*) \\ \nabla_{\Theta} h(\Theta_*)^{\top} & - \phi \end{matrix} \right] z.
    \]
    Note that $g = \nabla_{\Theta} h$, therefore $\nabla_{\Theta} g(\Theta_*) = \nabla_{\Theta}^2 h(\Theta_*)$ and by using the equilibrium relationships in Proposition~\ref{prop:equilibrium}, we get $\tfrac{d}{d\xi}z \simeq Jz$. 
    
    The study of the stochastic differential equation follows the steps as previously. The critical part is about $J$ being Hurwitz. Denote by $\mu(A)$ the largest real part of the eigenvalues of $A \in \R^{n \times n}$. Then we get 
    \begin{multline}
        \label{eq:mu}
        \mu(J) \leq \mu\left( \frac{J + J^{\top}}{2} \right) \\
        \leq \max\left(\mu\left(- \nabla_{\Theta}^2 f(\Theta_*) - \frac{h(\Theta_*)}{\phi} \nabla_{\Theta}^2 h(\Theta_*) \right), -\phi\right).
    \end{multline}
    Since $h \geq 0$ and both $f + h(\Theta_*) \phi^{-1} h$ is strongly convex around $\Theta_*$, then $J$ is Hurwitz and the conclusion holds.
\end{proof}

\subsubsection*{Connection with accuracy and robustness}

The accuracy (as defined in Section~2) of the closed-loop with $\phi = 0$ is improved since there are fewer equilibrium points compared to the vanilla PINN ($u > 0$ is constant), but its robustness depends only on the curvature of $f$ and the covariances $Q_w$ and $Q_n$. 

A nonzero leak parameter $\phi>0$ has two distinct and antagonistic
effects on the dynamics. First, it modifies the equilibrium point
$z^*$ itself. Indeed, unless $h(\Theta_*) = 0$ one has $I^*\neq 0$ and the primal equilibrium $\Theta_*$ differs from the equilibrium of the pure-integral system. In other words, the introduction of the leak may introduce a \emph{steady bias} in the estimator, since the controller no longer enforces the constraint $h(\Theta)=0$ exactly at steady state. Therefore, its accuracy is decreased.

Second, increasing $\phi$ improves the robustness of the dynamics to
noise. At the linearized level, the spectral abscissa $\mu(J)$ (from equation~\eqref{eq:mu}) is nonincreasing in $\phi$ for low $\phi$ if both $f$ and $h$ are strongly convex around $\Theta_*$. It follows that increasing $\phi$ close to $0$
shifts all eigenvalues of $J$ strictly to the left, enhancing the local
damping of the dynamics. In turn, the steady-state covariance $S$ decreases, meaning that the variance of $\Theta$ decreases as the leak increases. Moreover, the leakage ensures that the control action remains positive and lower than it would be with the pure integrator. Consequently, $\tr(Q_w) = \tr(Q_f) + u \tr(Q_g)$ is lower asymptotically in the leaky case, decreasing even further $\tr(S)$. All these effects suggest that it is enlarging the region of attraction, ensuring convergence in more scenarios.

The leak parameter $\phi$, therefore, introduces a classical
bias--variance trade--off: small values of $\phi$ enforce the physics
constraint accurately but yield larger sensitivity to stochastic
perturbations, while larger $\phi$ improves robustness at the price of a nonzero steady bias. 

\begin{remark}
    Note also that in the mismatch case with a discrepancy between the data and the physics cost, there is no stable equilibrium for the pure integrator, whereas the leaky one still allows for this possibility.
\end{remark}

\section{Numerical Simulations}

In this section, we consider a toy example, and we estimate the solution to the following differential system:
\begin{equation}
    \begin{array}{ll}
        \dot{x}(t) = \left[ \begin{matrix} -5 & a \\ 0 & -2 \end{matrix} \right] x(t) + \left[ \begin{matrix} 1 \\ 2 \end{matrix} \right] \left( \sin(t) + \cos(t) \right), & t \in [0, 6], \\
        x(0) = \left[\begin{matrix} 1 & 0 \end{matrix}\right]^{\top}, & \Gamma_1 = \{0\}, \\
    \end{array}
\end{equation}
The parameter $a$ accounts for model mismatch. For data-generation, we consider $a=3$, and we measure $\tilde{x}(t) = x(t) + n(t)$ at some instant $t \in \Gamma_2 = \{ 0.5, 1, 1.5, \dots, 6\}$ with $n \sim \mathcal{N}(0, 0.01)$.

We ran the previous algorithms with the hyperparameters from the code available on GitHub\footnote{\texttt{https://github.com/mBarreau/Controlled-PINN}} and $k_I = 1$. We evaluated two cases: the \emph{match} case, where the physics known matches the data ($a = 3$), and the \emph{mismatch} case, where the supposed underlying model is slightly incorrect ($a = 2.7$). Since this is a toy example, we investigate whether the proposed control methodologies correctly handle these cases and compare them with the vanilla PINN. The results of $20$ trials with different seeds are presented in Figure~\ref{fig:analysis}. The reconstructed signals are denoted $\hat{x}$ and the normalized error is defined as $e = \sqrt{\int_0^6 \| x(t) - \hat{x}(t) \|^2 dt} \cdot \| x \|_{\Ltwo}^{-1}$.

\begin{figure}
    \centering
    \includegraphics[width=\linewidth]{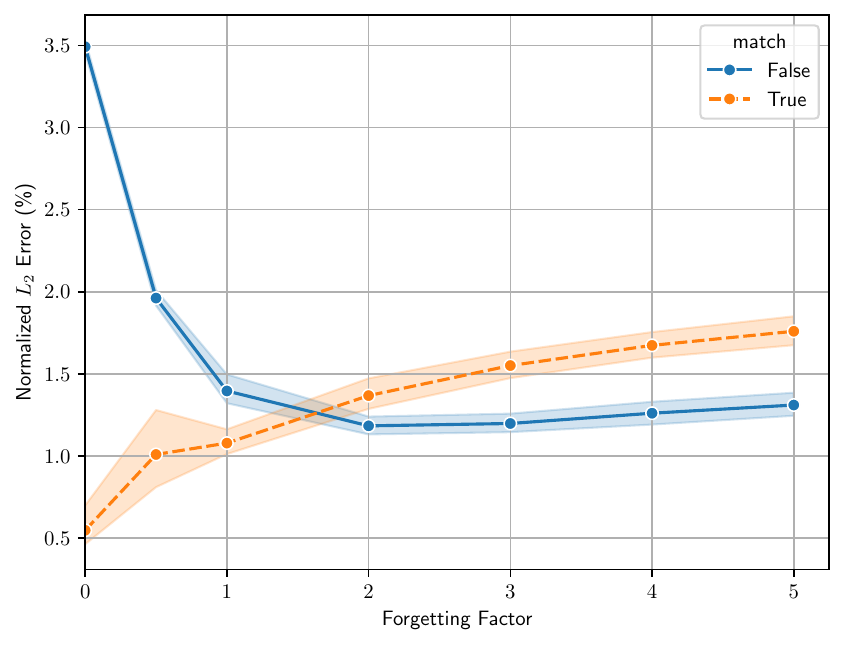}
    \caption{Accuracy of the reconstruction as a function of the forgetting factor in the match and mismatch cases.}
    \label{fig:analysis}
\end{figure}

For a zero forgetting factor, we apply the classic primal-dual methodology (pure integrator). In the matched case (dash-orange), we can see that the normalized error is very small, and the variance is also small ($0.55 \pm 0.28 \%$). With vanilla PINN and a constant $\lambda = 1$, we get a normalized error of $0.61 \pm 0.46 \%$, which is also good. Indeed, both of these results indicate a good fit and are likely close to the optimal one, which could be achieved by considering the approximation limitations of the finite-width neural network ($N < \infty$). We observe that the pure integrator training exhibits a lower variance than the vanilla PINN, which strengthens its robustness properties. However, in the mismatch case (solid blue), both methodologies fail to accurately estimate the real solution with $3.5 \pm 0.1\%$ normalized error for the integral control and $4.1 \pm 0.2\%$ for the vanilla PINN. Note, however, that the variance is still reduced by half when considering the integrator. This shows that the original method and the primal-dual one are comparable in terms of accuracy for that case; however, the controlled version exhibits a more robust behavior.

For the leaky integrator, in the nominal case with model match, the forgetting factor $\phi$ tends to deteriorate the accuracy, as evident from the increasing dash-orange curve. Unless around $\phi = 0.5$, we can see that the variance is constant. This confirms the local stability analysis results proposed previously. In the unmatched case, it is clear that the leaky integrator considerably improves performance. The normalized error decreases significantly, attaining its minimum around $\phi = 2$ with $1.2 \pm 0.13\%$. The forgetting factor provided guidance towards the physical model, but measurement fitting was more important in case of discrepancy, leading to a consistent estimation. This can be seen in Figure~\ref{fig:weights}, where the control action is similar for both controllers initially, but the leaky integrator reduces as the physics loss starts to be small, ensuring good fitting with the data, while the other one continues to increase.

\begin{figure}
    \centering
    \includegraphics[width=1\linewidth]{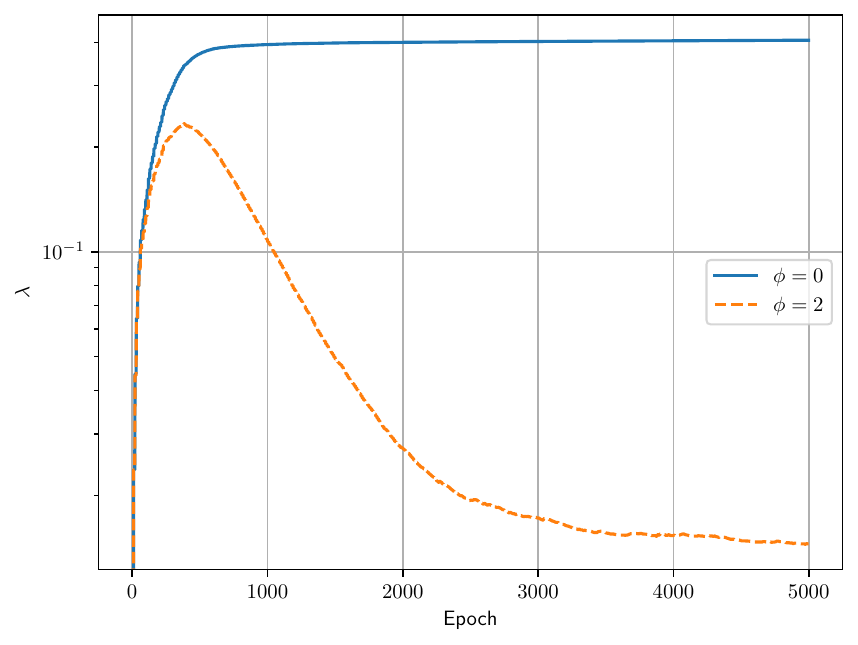}
    \caption{Evolution of the control signal with and without integral leakage.}
    \label{fig:weights}
\end{figure}

Overall, it appears that the integral controller indeed exhibits signs of improved robustness, which was the primary purpose of the method. On the other hand, the leaky integral introduces the classic bias-variance trade-off. A large forgetting factor does not guarantee the most accurate estimation; however, it produces very good results, even in cases of model mismatch. That means that the classical primal-dual method is appropriate for the \emph{forward} estimation task (i.e., the problem is well-posed and there is a unique solution to \eqref{eq:PDE}) while the leaky integral controller seems more appropriate for the \emph{inverse} problem (i.e., when we try to find an underlying model behind data) since it can handle model mismatch.

\section{Conclusion}

We introduced a control-theoretic viewpoint on the training of PINNs, interpreting the learning dynamics as a feedback interconnection between the data loss and the physics residual. This perspective reveals that the primal-dual formulation corresponds to an integral controller, while adding a forgetting factor yields a leaky-integral controller.

The stability analysis shows that these two controllers exhibit different behaviors. The pure integrator guarantees exact enforcement of the physics at equilibrium, but its lack of dissipation makes it sensitive to sampling noise. In contrast, the leaky controller introduces an equilibrium bias while increasing the damping in the closed-loop dynamics, resulting in a reduced steady-state variance. This provides a theoretical explanation for the accuracy--robustness trade-off observed
experimentally on a toy example and for the superior performance of the leaky controller in mismatched or noisy settings. The integral controller performs well in the matched case, indicating that it is well-suited for forward problems where the physical model is accurate and reliable. In contrast, the leaky integrator provides more consistent results under model mismatch, indicating its potential relevance for inverse problems. In both cases, the variance of the reconstruction error is reduced compared to standard PINNs.

This work represents an initial step toward more interpretable and reliable training algorithms for PINNs. Future research will focus on extending the analysis to more complex systems and interpreting other successful PINN strategies through the lens of control theory.

\bibliography{ifacconf}             

@book{Goodfellow-et-al-2016,
    title={Deep Learning},
    author={Goodfellow, I. and Bengio, Y. and Courville, A.},
    publisher={MIT Press},
    year={2016}
}

@inproceedings{choromanska2015loss,
  title={The loss surfaces of multilayer networks},
  author={Choromanska, A. and Henaff, M. and Mathieu, M. and Arous, G. B. and LeCun, Y.},
  booktitle={Artificial intelligence and statistics},
  year={2015},
  organization={PMLR}
}

@article{son2023sobolev,
  title={Sobolev training for physics informed neural networks},
  author={Son, H. and Jang, J. W. and Han, W. J. and Hwang, H. J.},
  journal={Communications in Mathematical Sciences},
  year={2023}
}

@article{shukla2021parallel,
  title={Parallel physics-informed neural networks via domain decomposition},
  author={Shukla, K. and Jagtap, A. D. and Karniadakis, G. E.},
  journal={Journal of Computational Physics},
  year={2021},
  publisher={Elsevier}
}

@article{guo2023pre,
  title={Pre-training strategy for solving evolution equations based on physics-informed neural networks},
  author={Guo, J. and Yao, Y. and Wang, H. and Gu, T.},
  journal={Journal of Computational Physics},
  year={2023},
  publisher={Elsevier}
}

@article{jagtap2020adaptive,
  title={Adaptive activation functions accelerate convergence in deep and physics-informed neural networks},
  author={Jagtap, A. D. and Kawaguchi, K. and Karniadakis, G. E.},
  journal={Journal of Computational Physics},
  year={2020},
  publisher={Elsevier}
}

@article{psaros2022meta,
  title={Meta-learning {PINN} loss functions},
  author={Psaros, A. F. and Kawaguchi, K. and Karniadakis, G. E.},
  journal={Journal of computational physics},
  year={2022},
  publisher={Elsevier}
}

@article{wang2020understanding,
  title={Understanding and mitigating gradient pathologies in physics-informed neural networks},
  author={Wang, S. and Teng, Y. and Perdikaris, P.},
  journal = {SIAM Journal on Scientific Computing},
year = {2021},
}

@article{maddu2022inverse,
  title={Inverse {Dirichlet} weighting enables reliable training of physics-informed neural networks},
  author={Maddu, S. and Sturm, D. and M{\"u}ller, C. L. and Sbalzarini, I. F.},
  journal={Machine Learning: Science and Technology},
  year={2022},
  publisher={IOP Publishing}
}

@article{wang2020and,
  title={When and why {PINNs} fail to train: A neural tangent kernel perspective},
  author={Wang, S. and Yu, X. and Perdikaris, P.},
    journal={Journal of Computational Physics},
  year={2022},
  publisher={Elsevier}
}

@article{sirignano2018dgm,
  title={{DGM}: A deep learning algorithm for solving partial differential equations},
  author={Sirignano, J. and Spiliopoulos, K.},
  journal={Journal of computational physics},
  year={2018},
  publisher={Elsevier}
}

@article{weitenberg2018robust,
  title={Robust decentralized secondary frequency control in power systems: Merits and tradeoffs},
  author={Weitenberg, Erieke and Jiang, Yan and Zhao, Changhong and Mallada, Enrique and De Persis, Claudio and D{\"o}rfler, Florian},
  journal={IEEE Transactions on Automatic Control},
  volume={64},
  number={10},
  pages={3967--3982},
  year={2018},
  publisher={IEEE}
}

@book{liu2019stochastic,
  title={Stochastic stability of differential equations in abstract spaces},
  author={Liu, Kai},
  volume={453},
  year={2019},
  publisher={Cambridge University Press}
}

@book{aastrom1995pid,
title = "PID Controllers: Theory, Design, and Tuning",
author = "{\AA}str{\"o}m, \{Karl Johan\} and Tore H{\"a}gglund",
year = "1995",
language = "English",
isbn = "1-55617-516-7",
publisher = "ISA - The Instrumentation, Systems and Automation Society",
}

@inproceedings{bengio2009curriculum,
  title={Curriculum learning},
  author={Bengio, Y. and Louradour, J. and Collobert, R. and Weston, J.},
  booktitle={Proceedings of the 26th annual international conference on machine learning},
  year={2009}
}

@article{hornik1991approximation,
  title={Approximation capabilities of multilayer feedforward networks},
  author={Hornik, K.},
  journal={Neural networks},
  year={1991},
  publisher={Elsevier}
}

@book{bertsekas2014constrained,
  title={Constrained optimization and Lagrange multiplier methods},
  author={Bertsekas, Dimitri P},
  year={2014},
  publisher={Academic press}
}

@InProceedings{fioretto2020lagrangian,
author="Fioretto, F.
and Van Hentenryck, P.
and Mak, T. W. K.
and Tran, C.
and Baldo, F.
and Lombardi, M.",
title="Lagrangian Duality for Constrained Deep Learning",
booktitle="Machine Learning and Knowledge Discovery in Databases. Applied Data Science and Demo Track",
year="2021",
publisher="Springer International Publishing",
}

@inproceedings{nandwani2019primal,
  author={Nandwani, Y. and Pathak, A. and Singla, P.},
  editor    = {Wallach, H. M. and Larochelle, H. and Beygelzimer, A. and d'Alch{\'{e}}{-}Buc, F. and Fox, E. B. and Garnett, R.},
  title     = {A Primal Dual Formulation For Deep Learning With Constraints},
  booktitle = {Advances in Neural Information Processing Systems (NeurIPS)},
  year      = {2019}
}

@article{goemans1997primal,
  title={The primal-dual method for approximation algorithms and its application to network design problems},
  author={Goemans, M. X. and Williamson, D. P.},
  journal={Approximation algorithms for NP-hard problems},
  year={1997}
}

@article{goodfellow2014generative,
  title={Generative adversarial nets},
  author={Goodfellow, I. and Pouget-Abadie, J. and Mirza, M. and Xu, B. and Warde-Farley, D. and Ozair, S. and Courville, A. and Bengio, Y.},
  journal={Advances in neural information processing systems},
  year={2014}
}

@article{cerone2025new,
  title={A new framework for constrained optimization via feedback control of Lagrange multipliers},
  author={Cerone, V. and Fosson, S. M. and Pirrera, S. and Regruto, D.},
  journal={IEEE Transactions on Automatic Control},
  year={2025},
  publisher={IEEE}
}

@article{qu2018exponential,
  title={On the exponential stability of primal-dual gradient dynamics},
  author={Qu, G. and Li, N.},
  journal={IEEE Control Systems Letters},
  year={2018},
  publisher={IEEE}
}

@book{zaccarian2011modern,
  title={Modern anti-windup synthesis: control augmentation for actuator saturation},
  author={Zaccarian, Luca and Teel, Andrew R},
  year={2011},
  publisher={Princeton University Press}
}

@inproceedings{rathore2024challenges,
author = {Rathore, Pratik and Lei, Weimu and Frangella, Zachary and Lu, Lu and Udell, Madeleine},
title = {Challenges in training PINNs: a loss landscape perspective},
year = {2024},
publisher = {JMLR.org},
booktitle = {Proceedings of the 41st International Conference on Machine Learning},
articleno = {1715},
numpages = {33},
location = {Vienna, Austria},
series = {ICML'24}
}

@article{daskalakis2018limit,
  title={The limit points of (optimistic) gradient descent in min-max optimization},
  author={Daskalakis, C. and Panageas, I.},
  journal={Advances in neural information processing systems},
  year={2018}
}

@article{li2025milpinitializationsolvingparabolic,
      title={{MILP} initialization for solving parabolic {PDEs} with {PINNs}}, 
      author={Li, S. and Bragone, F. and Barreau, M. and Morozovska, K.},
      year={2025},
      journal={arXiv preprint arXiv:2501.16153},
}

@article{liu2024config,
  title={{ConFIG}: Towards Conflict-free Training of Physics Informed Neural Networks},
  author={Liu, Q. and Chu, M. and Thuerey, N.},
  journal={arXiv preprint arXiv:2408.11104},
  year={2024}
}

@article{baydin2017automatic,
  title={Automatic differentiation in machine learning: a survey},
  author={Baydin, A. G. and Pearlmutter, B. A. and Radul, A. A. and Siskind, J. M.},
  journal={The Journal of Machine Learning Research},
  year={2017},
  publisher={JMLR. org}
}

@article{kingma2015adam,
  title={Adam: A Method for Stochastic Optimization},
  author={Kingma, D. P. and Ba, J.},
  journal={CoRR},
  year={2014},
}

@inproceedings{glorot2010understanding,
  title={Understanding the difficulty of training deep feedforward neural networks},
  author={Glorot, X. and Bengio, Y.},
  booktitle={Proceedings of the thirteenth international conference on artificial intelligence and statistics},
  year={2010},
  organization={JMLR Workshop and Conference Proceedings}
}

@book{Hartman2002ODE,
  author    = {Hartman, P.},
  title     = {Ordinary Differential Equations},
  publisher = {SIAM},
  address   = {Philadelphia},
  year      = {2002},
  edition   = {2},
  doi       = {10.1137/1.9780898719222}
}

@article{karniadakis2021physics,
  title={Physics-informed machine learning},
  author={Karniadakis, G. E. and Kevrekidis, I. G. and Lu, L. and Perdikaris, P. and Wang, S. and Yang, L.},
  journal={Nature Reviews Physics},
  year={2021},
  publisher={Nature Publishing Group}
}

@article{elkabetz2021continuous,
  title={Continuous vs. discrete optimization of deep neural networks},
  author={Elkabetz, O. and Cohen, N.},
  journal={Advances in Neural Information Processing Systems},
  year={2021}
}

@book{ladyzhenskaia1968linear,
  title={Linear and quasi-linear equations of parabolic type},
  author={Ladyzhenskaia, Olga Aleksandrovna and Solonnikov, Vsevolod Alekseevich and Ural'tseva, Nina N},
  year={1968},
  publisher={American Mathematical Soc.}
}

@article{mcclenny2023self,
  title={Self-adaptive physics-informed neural networks},
  author={McClenny, Levi D and Braga-Neto, Ulisses M},
  journal={Journal of Computational Physics},
  volume={474},
  pages={111722},
  year={2023},
  publisher={Elsevier}
}

\end{document}